\documentclass{paper}

\usepackage{graphicx}
\usepackage{amsmath}
\usepackage{amsthm}
\usepackage{amssymb}
\usepackage{dsfont}
\usepackage{booktabs}
\usepackage{array}

\usepackage{bm}
\usepackage[lined,boxed,nofillcomment]{algorithm2e}

\newcolumntype{C}[1]{>{\centering\let\newline\\\arraybackslash\hspace{0pt}}m{#1}}
\newcolumntype{R}[1]{>{\centering\let\newline\\\arraybackslash\hspace{0pt}}m{#1}}
\newcommand{\IR}{{\hbox{I\kern-.15em R}}}
\newtheorem{theorem}{Theorem}[section]

\newtheorem{proposition}[theorem]{Proposition}
\newtheorem{corollary}[theorem]{Corollary}

\begin{document}

\title{Safe and Efficient Screening For Sparse Support Vector Machine}
\author{Zheng Zhao and Jun Liu}

\maketitle
\section{Sparse SVM in Primal Form}
Assume that $\mathbf{X}\in\IR^{m\times n}$ is a data set containing $n$ samples, $\mathbf{X}=\left(\mathbf{x}_1,\ldots,\mathbf{x}_n\right)$, and $m$ features, $\mathbf{X}=\left(\mathbf{f}_1^\top,\ldots,\mathbf{f}^\top_m\right)^\top$, and $\mathbf{y}=\left(y_1,\ldots,y_n\right)$ contains the class label of $n$ samples, and $y_i\in\left\{-1,+1\right\},~i=1,\ldots,n$. The primal form of the L1-regularized L2-Loss support vector machine (SVM) is defined as:
\begin{eqnarray}
\min\limits_{\boldsymbol{\xi},\mathbf{w}}{\frac{1}{2}\sum\limits^n_{i=1}{\xi_i^2}+\lambda||\mathbf{w}||_1}\label{eq:svml2r1_primal}\\
s.t.~~y_i\left(\mathbf{w}^\top\mathbf{x}_i+b\right)\ge 1-\xi_i,\nonumber\\
\xi_i\ge 0.\nonumber
\end{eqnarray}

Eq.~(\ref{eq:svml2r1_primal}) specifies a convex problem with a non-smooth $L_1$ regularizer, which enforce the solution to be sparse. Let $\boldsymbol{w}^\star(\lambda)$ be the optimal solution of Eq.~(\ref{eq:svml2r1_primal}) for a given $\lambda$. All the features with nonzero values in $\boldsymbol{w}^\star(\lambda)$ are called active features, and the other features are called inactive.

\section{Sparse SVM in Dual}
The Lagrangian multiplier~\cite{boyd-vand-2004-book} of the problem defined in Eq.~(\ref{eq:svml2r1_primal}) is:
\begin{eqnarray}
L\left(\mathbf{w},b,\boldsymbol{\xi,\alpha,\mu}\right)&=&\frac{1}{2}\sum\limits^n_{i=1}{\xi_i^2}+\lambda||\mathbf{w}||_1\label{eq:lp_org}\\
&-&\sum\limits_{i=1}^{n}\alpha_i\left(~y_i\left(\mathbf{w}^\top\mathbf{x}_i+b\right)-1+\xi_i\right)\nonumber\\
&-&\sum\limits_{i=1}^{n}{\mu_i\xi_i}~.\nonumber
\end{eqnarray}
The corresponding Karush-Kuhn-Tucker (KKT) conditions~\cite{boyd-vand-2004-book} are:
\begin{eqnarray}
\xi_i&\ge& 0\\
\alpha_i&\ge& 0\\
\mu_i &\ge& 0\\
y_i\left(\mathbf{w}^\top\mathbf{x}_i+b\right)-1+\xi_i&\ge& 0\\
\alpha_i\left(~y_i\left(\mathbf{w}^\top\mathbf{x}_i+b\right)-1+\xi_i\right)&=& 0\\
\xi_i\mu_i &=& 0
\end{eqnarray}
By defining $L\left(\mathbf{w}\right)$, $L\left(\xi_i\right)$, and $L\left(b\right)$ as:
\begin{eqnarray}
L\left(\mathbf{w}\right)&=&\lambda||\mathbf{w}||_1-\sum\limits_{i=1}^{n}{\alpha_i y_i\mathbf{w}^\top\mathbf{x}_i},\\
L\left(\xi_i\right)&=&\frac{1}{2}\xi_i^2-\alpha_i\xi_i-\mu_i\xi_i,\\
L\left(b\right)&=&\sum\limits_{i=1}^{n}{\alpha_i y_i b},
\end{eqnarray}
The Eq.~(\ref{eq:lp_org}) can be reformulated as:
\begin{equation}
L\left(\mathbf{w},b,\boldsymbol{\xi,\alpha,\mu}\right)= L\left(\mathbf{w}\right)+\sum\limits_{i=1}^{n}L\left(\xi_i\right)+ L\left(b\right)+\sum\limits_{i=1}^{n}\alpha_i.
\end{equation}
The minimum of $L\left(\mathbf{w},b,\boldsymbol{\xi,\alpha,\mu}\right)$ can be expressed as:
\begin{eqnarray}
\inf\limits_{\mathbf{w},b,\boldsymbol{\xi}}L\left(\mathbf{w},b,\boldsymbol{\xi,\alpha,\mu}\right)&=& \inf\limits_{\mathbf{w}}L\left(\mathbf{w}\right)+\sum\limits_{i=1}^{n}\inf\limits_{\xi_i}L\left(\xi_i\right)+ \inf\limits_{b}L\left(b\right)+\sum\limits_{i=1}^{n}\alpha_i.\label{eq:Lp_1}
\end{eqnarray}

Since the problem defined in Eq.~(\ref{eq:svml2r1_primal}) is convex and the optimal value of the objective function is achievable, the strong duality condition holds. Therefore, $\inf\limits_{\mathbf{w}}L\left(\mathbf{w}\right)<-\infty$, $\inf\limits_{\xi_i}L\left(\xi_i\right)<-\infty$, $\inf\limits_{b}L\left(b\right)<-\infty$. By applying standard optimization technique, we can obtain their minimum. 

\subsection*{The minimum of $L\left(\mathbf{w}\right)$}
{The minimum of $L\left(\mathbf{w}\right)$} is given by the following equation:
\begin{eqnarray}
\inf\limits_{\mathbf{w}}L\left(\mathbf{w}\right)=0,~~\mbox{when}~\|\hat{\mathbf{f}}_j^\top\boldsymbol{\alpha}\|\le \lambda,~~j=1,\ldots,m\label{eq:L_w_v}.
\end{eqnarray}
In the preceding equation $\hat{\mathbf{f}}_j=\mathbf{Y}\mathbf{f}_j$, and $\mathbf{Y}$ is a diagonal matrix and $Y_{i,i}=y_i,~i=1,\ldots,n$. Also, the following equation holds when minimum is achieved:
\begin{equation}\displaystyle
\boldsymbol{\alpha}^\top\hat{\mathbf{f}}_j=
{\left\{ {\begin{array}{*{20}l}
   {\mbox{sign}\left(w_j\right)\lambda},~~\mbox{if}~w_j\neq 0  \\
   {\left[ -\lambda, +\lambda\right]},~~~~\mbox{if}~w_j= 0  \\
\end{array}} \right.~~~j=1,\ldots,m}
\end{equation}

\subsection*{The minimum of $L\left(\xi_i\right)$}
{The minimum of $L\left(\xi_i\right)$}  is given by the following equation:
\begin{eqnarray}
\inf\limits_{\xi_i}L\left(\xi_i\right)=-\frac{1}{2}\alpha_i^2,~~\mbox{when}~~\xi_i=\alpha_i,~\mu_i=0~~i=1,\ldots,n\label{eq:L_x}
\end{eqnarray}

\subsection*{The minimum of $L\left(b\right)$}
{The minimum of $L\left(b\right)$}  is given by the following equation:
\begin{eqnarray}
\inf\limits_{b}L\left(b\right)=0,~~\mbox{when}~~\sum\limits_{i=1}^{n}{\alpha_i y_i}=0\label{eq:L_b}
\end{eqnarray}

\subsection*{The Dual}
By substituting Equations~(\ref{eq:L_w_v}), (\ref{eq:L_x}), and (\ref{eq:L_b}) into Eq.~(\ref{eq:Lp_1}), the dual of the L1-regularized L2-Loss SVM can be expressed as the following equation:
\begin{eqnarray}
\min\limits_{\boldsymbol{\alpha}}\|\boldsymbol{\alpha}-\mathbf{1}\|_2^2\\
s.t.~~\|\hat{\mathbf{f}}_j^\top\boldsymbol{\alpha}\|\le \lambda,~~j=1,\ldots,m\nonumber\\
\sum\limits_{i=1}^{n}{\alpha_i y_i}=0\nonumber\\
\boldsymbol{\alpha}\succcurlyeq\mathbf{0}\nonumber
\end{eqnarray}

By defining $\boldsymbol{\alpha}=\lambda\boldsymbol{\theta}$, the preceding equation can be reformulated as:
\begin{eqnarray}
\min\limits_{\boldsymbol{\theta}}||\boldsymbol{\theta}-\frac{\mathbf 1}{\lambda}||_2^2\label{eq:dual-2}\\
s.t.~~\|\hat{\mathbf{f}}_j^\top\boldsymbol{\theta}\|\le 1,~~j=1,\ldots,m\nonumber\\
\sum\limits_{i=1}^{n}{\theta_i y_i}=0\nonumber\\
\boldsymbol{\theta}\succcurlyeq\mathbf{0}\nonumber
\end{eqnarray}

\section{The Relationship between Primal and Dual Variables}
In the primal formulation for the L1-regularized L2-loss SVM, the primal variables are $b$, $\mathbf{w}$, and $\boldsymbol{\xi}$. And in the dual formulation, the dual variables are $\boldsymbol{\alpha}$ and $\boldsymbol{\mu}$. When $b$ and $\mathbf{w}$ is known $\boldsymbol{\xi}$, $\boldsymbol{\alpha}$, and $\boldsymbol{\theta}$ can be obtained as:
\begin{equation}
\mu_i=0,~\xi_i=\alpha_i=\lambda\theta_i=\max\left(0,1-y_i\left(\mathbf{w}^\top\mathbf{x}_i+b\right)\right),~~~i=1,\ldots,n.\label{eq:primal-dual-relation}
\end{equation}

The relationship between $\boldsymbol{\alpha}$ and $\mathbf{w}$  can be expressed as:
\begin{equation}\displaystyle
\boldsymbol{\alpha}^\top\hat{\mathbf{f}}_j =
{\left\{ {\begin{array}{*{20}l}
   {\mbox{sign}\left(w_j\right)\lambda},~~\mbox{if}~w_j\neq 0  \\
   {\left[ -\lambda, +\lambda\right]},~~~~\mbox{if}~w_j= 0  \\
\end{array}} \right.,~~~j=1,\ldots,m}
\end{equation}

The relationship between $\boldsymbol{\theta}$ and $\mathbf{w}$  can be expressed as:
\begin{equation}\displaystyle
\boldsymbol{\theta}^\top\hat{\mathbf{f}}_j=
{\left\{ {\begin{array}{*{20}l}
   {\mbox{sign}\left(w_j\right)},~~~~\mbox{if}~w_j\neq 0  \\
   {\left[ - 1, + 1\right]},~~~~\mbox{if}~w_j= 0  \\
\end{array}} \right.,~~~j=1,\ldots,m}\label{eq:is_active01}
\end{equation}

\section{Computing $\lambda_{\max}$}
$\lambda_{\max}$ is defined as the smallest value of $\lambda$ that results $\mathbf{w}=\mathbf{0}$ when it used in Eq.~(\ref{eq:svml2r1_primal}). When the input is given, it can be obtained in a closed form.

The L1-regularized L2-Loss SVM in Eq. (\ref{eq:svml2r1_primal}) can be rewritten in an unconstrainted form as:
\begin{equation}\label{eq:svml2r1_primal:squared:hinge}
\min h(\mathbf w, b) + \lambda \|\mathbf w \|_1,
\end{equation}
where $h(\mathbf w, b) =\frac{1}{2} \sum_{i=1}^n \max( 1- y_i   (\mathbf{w}^\top\mathbf{x}_i+b ), 0)^2 $. The derivative of $h(\mathbf w, b)$ with regard to $\mathbf w$ and $b$ can be computed as:
\begin{equation} \label{eq:derivative:w}
h'_{\mathbf w}(\mathbf w, b) = -\sum_{i=1}^n \max( 1- y_i   (\mathbf{w}^\top\mathbf{x}_i+b ), 0) y_i \mathbf x_i
\end{equation}

\begin{equation} \label{eq:derivative:b}
h'_{b}(\mathbf w, b) = -\sum_{i=1}^n \max( 1- y_i   (\mathbf{w}^\top\mathbf{x}_i+b ), 0) y_i 
\end{equation}

By the definition of $\lambda_{\max}$, when $\lambda$ is larger than $\lambda_{\max}$, $\mathbf w^*=0$, therefore,
$$h'_{b}(\mathbf 0, b^*) = -\sum_{i=1}^n \max( 1- y_i  b^*, 0) y_i =0, $$
and
$$ \|h'_{\mathbf{w}}( \mathbf 0, b^*) \|_{\infty} = \|\sum_{i=1}^n \max( 1- y_i  b^*, 0) y_i \mathbf x_i \|_{\infty} \le \lambda, $$

This leads to the result:
$$b^* = \frac{(n_{+} - n_{-})}{n},$$
where $n_{+}$ and $n_{-}$ denote the number of positive and negative samples, respectively. Since $\lambda_{\max} = \|\sum_{i=1}^n \max( 1- y_i  b^*, 0) y_i \mathbf x_i \|_{\infty}.$ It is easy to verify that $b^* \in [-1, 1]$, thus $\max( 1- y_i  b^*, 0)= 1- y_i  b^* $. Therefore,

\begin{equation}
\lambda_{\max} = \left\|\sum_{i=1}^n \left(y_i-\frac{n_+-n_-}{n}\right)  \mathbf{x}_i\right\|_{\infty}.
\end{equation}

\section{The First Feature(s) to Enter Into the Model}

Denote $\mathbf m =\sum_{i=1}^n \left(y_i-\frac{n_+-n_-}{n}\right)  \mathbf{x}_i$. The first feature to enter the model is the one corresponding to the element with the largest magnitude in $\mathbf m$.

\section{Screening Rule Based on Dual Variable $\boldsymbol{\theta}$} Eq.~(\ref{eq:is_active01}) shows that the necessary condition for a feature $\mathbf{f}$ to be active in the optimal solution is $|\boldsymbol{\theta}^\top\hat{\mathbf{f}}|=1$, where $\hat{\mathbf{f}}=\mathbf{Yf}$ and $\mathbf{Y}$ is a diagonal matrix and $Y_{i,i}=y_i,~i=1,\ldots,n$. This condition can be used to develop a screening rule for the L1-regularized L2-Loss SVM to speedup its training. More specifically, given $\lambda$, we can compute the upper bound of the value of $|\boldsymbol{\theta}^\top\hat{\mathbf{f}}|$, and remove all the features with its upper bound values being less than 1, which are garanteed to be inactive for the given $\lambda$. If the cost of computing this upper bound is low, we can use it to speedup the training process by removing many features. To bound value of $|\boldsymbol{\theta}^\top\hat{\mathbf{f}}|$, we need to first construct a closed convex set $\mathbf{K}$ that contains $\boldsymbol{\theta}$. Then we can obtain the upper bound value by maximizing $|\boldsymbol{\theta}^\top\hat{\mathbf{f}}|$ over $\mathbf{K}$. We first study how to construct the convex set $\mathbf{K}$.

\subsection{Constructing The Convex Set $\mathbf{K}$}
In the following, we construct a closed convex set $\mathbf{K}$ based on Eq.~(\ref{eq:dual-2}) and the variational inequality~\cite{lion-stam-cpam-67}. We first introduce the variational inequality for convex optimization. 

\begin{proposition}\label{prop:vi}
Let $\boldsymbol{\theta}$ be a solution to the optimization problem:
\begin{equation}
\min g(\boldsymbol{\theta}),~~s.t.~~\boldsymbol{\theta}\in\mathbf K
\end{equation}
where $g$ is continuously differentiable and $\mathbf K$ is closed and convex. Then $\boldsymbol{\theta}^\star$ is a solution of the variational inequality problem:
\begin{equation}
\nabla g\left(\boldsymbol{\theta}^\star\right)^\top\left(\boldsymbol{\theta}-\boldsymbol{\theta}^\star\right)\ge 0,~~~\forall \boldsymbol{\theta}\in\mathbf{K}.\label{eq:vi}
\end{equation}
\end{proposition}

The proof of this proposition can be found in~\cite{lion-stam-cpam-67}.

Given $\lambda_2<\lambda_{max}$, we assume that there is a $\lambda_1$, such that $\lambda_{max} \ge \lambda_1 > \lambda_2$ and its corresponding solution $\boldsymbol{\theta}_1$ is known\footnote{When $\lambda_1=\lambda_{max}$, $\boldsymbol{\theta}_1$ can be easily obtained by using Eq.~(\ref{eq:primal-dual-relation}).}. The reason to introduce $\lambda_1$ is that when $\lambda_1$ is close to $\lambda_2$ and $\boldsymbol{\theta}_1$ is known, this can help us to construct a tighter convex set that contains $\boldsymbol{\theta}_2$ to bound the value of $|\boldsymbol{\theta}_2^\top\hat{\mathbf{f}}|$ in a better way.

Let $\boldsymbol{\theta}_1$ and $\boldsymbol{\theta}_2$ be the optimal solutions of the problem defined in Eq.~(\ref{eq:dual-2}) for $\lambda_1$ and $\lambda_2$, respectively. Assume that $\lambda_1 > \lambda_2$, and $\boldsymbol{\theta}_1$ is known. The following results can be obtained by applying Proposition~\ref{prop:vi} to the objective function defined in Eq.~(\ref{eq:dual-2}) for $\boldsymbol{\theta}_1$ and $\boldsymbol{\theta}_2$, respectively.

\begin{eqnarray}
\left(\boldsymbol{\theta}_1-\frac{\mathbf{1}}{\lambda_1}\right)^\top\left(\boldsymbol{\theta}-\boldsymbol{\theta}_1\right)\ge 0\label{eq:iv-l}\\
\left(\boldsymbol{\theta}_2-\frac{\mathbf{1}}{\lambda_2}\right)^\top\left(\boldsymbol{\theta}-\boldsymbol{\theta}_2\right)\ge 0\label{eq:iv-b}
\end{eqnarray}

By substituting $\boldsymbol{\theta}=\boldsymbol{\theta}_2$ into Eq.~(\ref{eq:iv-l}), and $\boldsymbol{\theta}=\boldsymbol{\theta}_1$ into Eq.~(\ref{eq:iv-b}), the following equations can be obtained.

\begin{eqnarray}
\left(\boldsymbol{\theta}_1-\frac{\mathbf{1}}{\lambda_1}\right)^\top\left(\boldsymbol{\theta}_2-\boldsymbol{\theta}_1\right)\ge 0\label{eq:iv-line}\\
\left(\boldsymbol{\theta}_2-\frac{\mathbf{1}}{\lambda_2}\right)^\top\left(\boldsymbol{\theta}_2-\boldsymbol{\theta}_1\right)\le 0\label{eq:iv-ball}
\end{eqnarray}

In the preceding equations, $\boldsymbol{\theta}_1$, $\lambda_1$, and $\lambda_2$ are known. Therefore, Eq.~(\ref{eq:iv-line}) defines a $n$ dimensional halfspace and Eq.~(\ref{eq:iv-ball}) defines a $n$ dimensional hyperball. Since $\boldsymbol{\theta}_2$ needs to satisfy both equations, it must reside in the region formed by the intersection of the halfspace and the hyperball. Obviously, this region is a closed convex set, and can be used as the $\mathbf{K}$ to bound $|\boldsymbol{\theta}_2^\top\hat{\mathbf{f}}|$. 

Fig.~\ref{fig:BND-1} shows an example of the $\mathbf{K}$ in a two dimensional space. In the figure, $\left(\boldsymbol{\theta}_1-\frac{\mathbf{1}}{\lambda_1}\right)^\top\left(\boldsymbol{\theta}_2-\boldsymbol{\theta}_1\right)=0$ defines the blue line. And $\left(\boldsymbol{\theta}_2-\frac{\mathbf{1}}{\lambda_2}\right)^\top\left(\boldsymbol{\theta}_2-\boldsymbol{\theta}_1\right)=0$ defines the red circle. And $\mathbf{K}$ is indicated by the shaded area.

\begin{figure}[htpb]\centering\scriptsize
\includegraphics[width=0.9\textwidth]{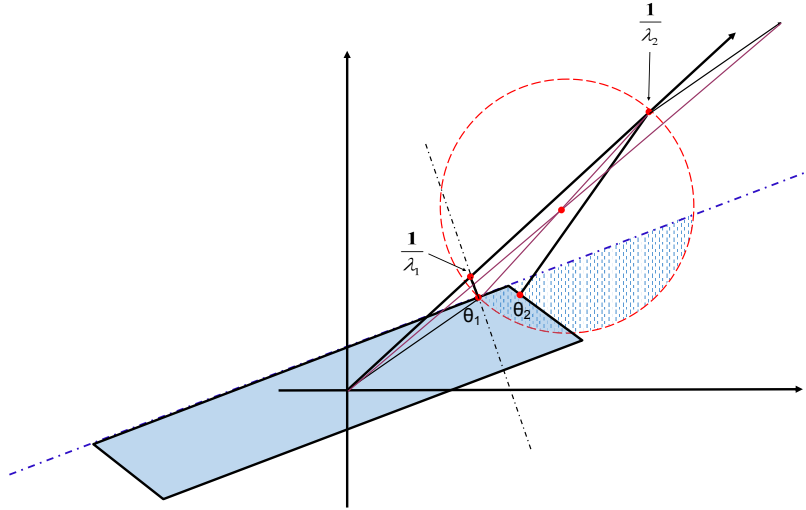}
\caption{The $\mathbf{K}$ in a 2D space. It is indicated by the shaded area.}\label{fig:BND-1}
\end{figure}

Besides the $n$ dimensional hyperball defined in Eq.~(\ref{eq:iv-ball}), it is possible to derive a series of hyperball by combining Eq.~(\ref{eq:iv-line}) and Eq.~(\ref{eq:iv-ball}). Assume that $\boldsymbol{\theta}^\star$ is the optimal solutions of Eq.~(\ref{eq:dual-2}) and $t\ge 0$, it is easy to verify that $\boldsymbol{\theta}^\star$ is also the optimal soultion of the following problem.
\begin{eqnarray}
\min\limits_{\boldsymbol{\theta}}\left\|\boldsymbol{\theta}-\left(t\frac{\mathbf 1}{\lambda}+\left(1-t\right)\boldsymbol{\theta}^\star\right)\right\|_2^2\label{eq:dual-2-ext}\\
s.t.~~\|\hat{\mathbf{f}}_j^\top\boldsymbol{\theta}\|\le 1,~~j=1,\ldots,m\nonumber\\
\sum\limits_{i=1}^{n}{\theta_i y_i}=0\nonumber\\
\boldsymbol{\theta}\succcurlyeq\mathbf{0}\nonumber
\end{eqnarray}
By applying Proposition~\ref{prop:vi} to the objective function defined in Eq.~(\ref{eq:dual-2-ext}) for $\boldsymbol{\theta}_1$, and $\boldsymbol{\theta}_2$, the following results can be obtained.
\begin{eqnarray}
\left(\boldsymbol{\theta}_1-\left(t_1\frac{\mathbf 1}{\lambda_1}+\left(1-t_1\right)\boldsymbol{\theta}_1\right)\right)^\top\left(\boldsymbol{\theta}-\boldsymbol{\theta}_1\right)\ge 0\label{eq:iv-l-ext}\\
\left(\boldsymbol{\theta}_2-\left(t_2\frac{\mathbf 1}{\lambda_2}+\left(1-t_2\right)\boldsymbol{\theta}_2\right)\right)^\top\left(\boldsymbol{\theta}-\boldsymbol{\theta}_2\right)\ge 0\label{eq:iv-b-ext}
\end{eqnarray}

Let $t=\frac{t_1}{t_2}\ge 0$. By substituting $\boldsymbol{\theta}=\boldsymbol{\theta}_2$ and $\boldsymbol{\theta}=\boldsymbol{\theta}_1$ into Eq.~(\ref{eq:iv-l-ext}) and Eq.~(\ref{eq:iv-b-ext}), respectively, and then combining the two obtained equations, the following equation can be obtained.

\begin{equation}
\mathbf{B}_t=\left\{\boldsymbol{\theta}_2:\left(\boldsymbol{\theta}_2-\mathbf{c}\right)^\top\left(\boldsymbol{\theta}_2-\mathbf{c}\right)\le l^2\right\}\label{eq:ball}
\end{equation}
\begin{equation}
\mathbf{c}=\frac{1}{2}\left(t\boldsymbol{\theta}_1-t\frac{\mathbf{1}}{\lambda_1}+\frac{\mathbf{1}}{\lambda_2}+\boldsymbol{\theta}_1\right),~l=\frac{1}{2}\left\|t\boldsymbol{\theta}_1-t\frac{\mathbf{1}}{\lambda_1}+\frac{\mathbf{1}}{\lambda_2}-\boldsymbol{\theta}_1\right\|_2\nonumber
\end{equation}

As the value of $t$ change from 0 to $\infty$, Eq.~(\ref{eq:ball}) generates a series of hyperball. When $t=0$, $\mathbf{c}=\frac{1}{2}\left(\frac{\mathbf{1}}{\lambda_2}+\boldsymbol{\theta}_1\right)$ and $l=\frac{1}{2}\|\frac{\mathbf{1}}{\lambda_2}-\boldsymbol{\theta}_1\|_2$. This corresponds to the hyperball defined by Eq.~(\ref{eq:iv-ball}). The following theorems provide some insights about the properties of the hyperballs generated by Eq.~(\ref{eq:ball}).

\begin{theorem} Let $\mathbf{a}=\frac{\boldsymbol{\theta}_1-\frac{\mathbf{1}}{\lambda_1}}{\left\|\boldsymbol{\theta}_1-\frac{\mathbf{1}}{\lambda_1}\right\|_2}$, the radius of the hyperball generated by Eq.~(\ref{eq:ball}) reaches it minimum when,
\begin{eqnarray}
t=1-\left(\frac{1}{\lambda_2}-\frac{1}{\lambda_1}\right)\mathbf{a}^\top\mathbf{1}.
\end{eqnarray}
Let $\mathbf{\hat{c}}$ be the center of the ball and $l$ be the radius, in this case,
\begin{eqnarray}
\mathbf{\hat{c}}=\frac{1}{2}\left(\frac{1}{\lambda_2}-\frac{1}{\lambda_1}\right)P_{\mathbf{a}}\left(\mathbf{1}\right)+\boldsymbol{\theta}_1,~~l=\frac{1}{2}\left(\frac{1}{\lambda_2}-\frac{1}{\lambda_1}\right)\left\|P_{\mathbf{a}}\left(\mathbf{1}\right)\right\|.
\end{eqnarray}\label{th:small-ball}
Here, $P_{\mathbf{u}}\left(\mathbf{v}\right)$ is a operator projects $\mathbf{v}$ to the null-space of $\mathbf{u}$:
\begin{equation}
P_{\mathbf{u}}\left(\mathbf{v}\right)=\mathbf{v}-\frac{\mathbf{v}^\top\mathbf{u}}{\|\mathbf{u}\|^2_2}\mathbf{u}.
\end{equation}
Since $\|\mathbf{a}\|_2=1$, $P_{\mathbf{a}}\left(\mathbf{1}\right)=\mathbf{1}-\left(\mathbf{a}^\top\mathbf{1}\right)\mathbf{a}$.
\end{theorem}

\begin{proof}
The theorem can be proved by minimizing the $r$ defined in Eq.~(\ref{eq:ball}).
\end{proof}

\begin{theorem}\label{th:ball-same}
Let the intersection of the hyperplane $\left(\boldsymbol{\theta}_1-\frac{\mathbf{1}}{\lambda_1}\right)^\top\left(\boldsymbol{\theta}_2-\boldsymbol{\theta}_1\right)= 0$ and the hyperball defined by Eq.~(\ref{eq:ball}) be $\mathbf{P}_t$. The following equation holds.
$$\mathbf{P}_{t_1}=\mathbf{P}_{t_2}, ~\mbox{for}~~\forall t_1, t_2 \ge 0, t_1\ne t_2.$$
\end{theorem}
\begin{proof}
The hyperballs defined in Eq.~(\ref{eq:ball}) can be rewritten in the form:
\begin{equation}
\mathbf{B}_t=\left\{\boldsymbol{\theta}_2:\left(\boldsymbol{\theta}_2-\frac{\mathbf{1}}{\lambda_2}\right)^\top\left(\boldsymbol{\theta}_2-\boldsymbol{\theta}_1\right)-t\left(\boldsymbol{\theta}_1-\frac{\mathbf{1}}{\lambda_1}\right)^\top\left(\boldsymbol{\theta}_2-\boldsymbol{\theta}_1\right)\le 0\right\}
\end{equation}
The intersect between $\mathbf{B}_t$ and $\left(\boldsymbol{\theta}_1-\frac{\mathbf{1}}{\lambda_1}\right)^\top\left(\boldsymbol{\theta}_2-\boldsymbol{\theta}_1\right)= 0$ is:
\begin{equation}
\mathbf{P}_t=\left\{\boldsymbol{\theta}_2:\left(\boldsymbol{\theta}_2-\frac{\mathbf{1}}{\lambda_2}\right)^\top\left(\boldsymbol{\theta}_2-\boldsymbol{\theta}_1\right)~\mbox{and}~\left(\boldsymbol{\theta}_1-\frac{\mathbf{1}}{\lambda_1}\right)^\top\left(\boldsymbol{\theta}_2-\boldsymbol{\theta}_1\right)= 0\right\}
\end{equation}
Since $\mathbf{P}_t$ is independent to $t$, we have $\mathbf{P}_{t_1}=\mathbf{P}_{t_2}, ~\mbox{for}~~\forall t_1, t_2 \ge 0, t_1\ne t_2.$
\end{proof}
This theorem shows that the intersection between the hyperball $\mathbf{B}_t$ and the hyperplane $\left(\boldsymbol{\theta}_1-\frac{\mathbf{1}}{\lambda_1}\right)^\top\left(\boldsymbol{\theta}_2-\boldsymbol{\theta}_1\right)= 0$ is the same for different $t$ values.

\begin{theorem}
Let the intersection of the half space $\left(\boldsymbol{\theta}_1-\frac{\mathbf{1}}{\lambda_1}\right)^\top\left(\boldsymbol{\theta}_2-\boldsymbol{\theta}_1\right)\ge 0$ and the hyperball defined by Eq.~(\ref{eq:ball}) be $\mathbf{Q}_t$. The following inequality holds.
$$\mathbf{Q}_{t_1}\subseteq\mathbf{Q}_{t_2}, ~\mbox{for}~~\forall t_1, t_2 \ge 0, t_1\le t_2.$$\label{th:ball-contain}
\end{theorem}
\begin{proof}
The intersect between $\mathbf{B}_t$ and $\left(\boldsymbol{\theta}_1-\frac{\mathbf{1}}{\lambda_1}\right)^\top\left(\boldsymbol{\theta}_2-\boldsymbol{\theta}_1\right) \ge 0$ is:
\begin{equation}
\mathbf{Q}_t=\left\{\boldsymbol{\theta}_2:\left(\boldsymbol{\theta}_2-\frac{\mathbf{1}}{\lambda_2}\right)^\top\left(\boldsymbol{\theta}_2-\boldsymbol{\theta}_1\right)\le t\left(\boldsymbol{\theta}_1-\frac{\mathbf{1}}{\lambda_1}\right)^\top\left(\boldsymbol{\theta}_2-\boldsymbol{\theta}_1\right)\right\}
\end{equation}
Since both $t$ and $\left(\boldsymbol{\theta}_1-\frac{\mathbf{1}}{\lambda_1}\right)^\top\left(\boldsymbol{\theta}_2-\boldsymbol{\theta}_1\right)$ are nonnegative, it is obvious that for $\forall t_1, t_2 \ge 0$ and $t_1\le t_2$, if $\boldsymbol{\theta}_2\in\mathbf{Q}_{t_1}$, we must have $\boldsymbol{\theta}_2\in\mathbf{Q}_{t_2}$.
\end{proof}

This theorem shows that the volume of $\mathbf{Q}_t$ becomes bigger when $t$ becomes bigger. And $\mathbf{Q}_{t_1}\subseteq\mathbf{Q}_{t_2}$ if $t_1\le t_2$.

Fig.~\ref{fig:BND-2} shows two circles in a 2D space. The circle with red color corresponds to the one obtained by setting $t_1=0$ in Eq.~(\ref{eq:ball}). And the circle with blue color corresponds to the one obtained by setting $t_2=1-\left(\frac{1}{\lambda_2}-\frac{1}{\lambda_1}\right)\mathbf{a}^\top\mathbf{1}$ in Eq.~(\ref{eq:ball}). It can be observed in the figure that the intersections of the two circles and the line $\left(\boldsymbol{\theta}_1-\frac{\mathbf{1}}{\lambda_1}\right)^\top\left(\boldsymbol{\theta}_2-\boldsymbol{\theta}_1\right)= 0$ are the same, and this is consistent with Theorem~\ref{th:ball-same}. Also since $t_1\le t_2$, $\mathbf{Q}_{t_1}\subseteq\mathbf{Q}_{t_2}$, which is consistent with Theorem~\ref{th:ball-contain}.

\begin{figure}[htpb]\centering\scriptsize
\includegraphics[width=0.9\textwidth]{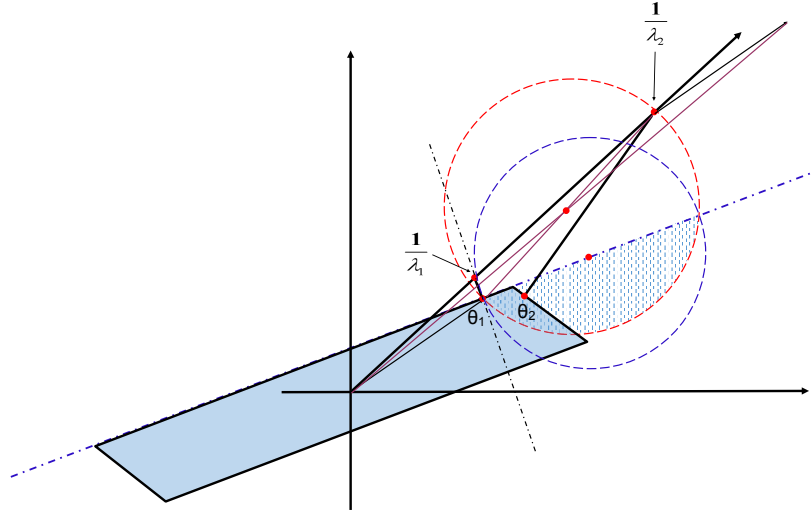}
\caption{The $\mathbf{K}$ in a 2D space when different $t$ values are used. The circle with red color corresponds to $t=0$, and the circle with blue color corresponds to $t=1-\left(\frac{1}{\lambda_2}-\frac{1}{\lambda_1}\right)\mathbf{a}^\top\mathbf{1}$.}\label{fig:BND-2}
\end{figure}

Thereom~\ref{th:ball-contain} suggests to use the $\mathbf{Q}_{t=0}$ to construct $\mathbf{K}$, since when $t=0$, the volumn of $\mathbf{Q}_{t}$ is minimized. The equality $\boldsymbol{\theta}^\top\mathbf{y}=0$ in Eq.~(\ref{eq:dual-2}) of the dual formulation can also be to further reduce the volumn of $\mathbf{K}$.

\begin{equation}
\mathbf{K}=\left\{\boldsymbol{\theta}_2:\left(\boldsymbol{\theta}_2-\mathbf{c}\right)^\top\left(\boldsymbol{\theta}_2-\mathbf{c}\right)\le l^2,~\left(\boldsymbol{\theta}_1-\frac{\mathbf{1}}{\lambda_1}\right)^\top\left(\boldsymbol{\theta}_2-\boldsymbol{\theta}_1\right) \ge 0,~\boldsymbol{\theta}_2^\top\mathbf{y}=0\right\}\nonumber
\end{equation}
\begin{equation}
where,~~\mathbf{c}=\frac{1}{2}\left(\frac{\mathbf{1}}{\lambda_2}+\boldsymbol{\theta}_1\right),~l=\frac{1}{2}\left\|\frac{\mathbf{1}}{\lambda_2}-\boldsymbol{\theta}_1\right\|_2\nonumber
\end{equation}
Let $\boldsymbol{\theta}_2=\mathbf{c}+\mathbf{r}$, $\mathbf{a}=\frac{\boldsymbol{\theta}_1-\frac{\mathbf{1}}{\lambda_1}}{\left\|\boldsymbol{\theta}_1-\frac{\mathbf{1}}{\lambda_1}\right\|_2}$, and $\mathbf{b}=\frac{1}{2}\left(\frac{\mathbf{1}}{\lambda_2}-\boldsymbol{\theta}_1\right)$, $\mathbf{K}$ can be rewritten as:

\begin{equation}
\mathbf{K}=\left\{\boldsymbol{\theta}_2:\boldsymbol{\theta}_2=\mathbf{c}+\mathbf{r},~\left\|\mathbf{r}\right\|^2\le \|\mathbf{b}\|^2,~\mathbf{a}^\top\left(\mathbf{b}+\mathbf{r}\right)\le 0,~\left(\mathbf{c+r}\right)^\top\mathbf{y}=0\right\}\label{eq:K}
\end{equation}
\begin{equation}
where,~\mathbf{a}=\frac{\boldsymbol{\theta}_1-\frac{\mathbf{1}}{\lambda_1}}{\left\|\boldsymbol{\theta}_1-\frac{\mathbf{1}}{\lambda_1}\right\|_2},~\mathbf{b}=\frac{1}{2}\left(\frac{\mathbf{1}}{\lambda_2}-\boldsymbol{\theta}_1\right),~\mathbf{c}=\frac{1}{2}\left(\frac{\mathbf{1}}{\lambda_2}+\boldsymbol{\theta}_1\right)\nonumber
\end{equation}

Theorem~\ref{th:ball-same} shows that when the value of $t$ varies, the intersection of the hyperball $\mathbf{B}_t$ and the hyperplane $\left(\boldsymbol{\theta}_1-\frac{\mathbf{1}}{\lambda_1}\right)^\top\left(\boldsymbol{\theta}_2-\boldsymbol{\theta}_1\right)= 0$ keeps unchange. This means that if the maximium value of $|\boldsymbol{\theta}^\top\hat{\mathbf{f}}|$ is achieved with a $\boldsymbol{\theta}$ in this area, no matter which $\mathbf{B}_t$ is used, the maximium value will be the same. This property can be used to simplify the computation. In Section~\ref{sec:bn0an0}, we will show that when the maximium value of $|\boldsymbol{\theta}^\top\hat{\mathbf{f}}|$ is achieved with a $\boldsymbol{\theta}$ on the intersection of the hyperball $\mathbf{B}_{t=0}$ and the hyperplane $\left(\boldsymbol{\theta}_1-\frac{\mathbf{1}}{\lambda_1}\right)^\top\left(\boldsymbol{\theta}_2-\boldsymbol{\theta}_1\right)= 0$, we can simplify the computation by switching to $\mathbf{B}_t$ with $t=1-\left(\frac{1}{\lambda_2}-\frac{1}{\lambda_1}\right)\mathbf{a}^\top\mathbf{1}$, which will enable us to derive a close form solution for the problem.

\subsection{Computing the Upper Bound}
Given the convex set $\mathbf{K}$ defined in Equation~(\ref{eq:K}), the maximum value of $\left|\boldsymbol{\theta}_2^\top\hat{\mathbf{f}}\right|=\left|\left(\mathbf{c+r}\right)^\top\hat{\mathbf{f}}\right|$ can be computed by solving the following optimization problem:
\begin{eqnarray}
&\max{\left|\left(\mathbf{c+r}\right)^\top\hat{\mathbf{f}}\right|}\label{eq:bound0}\\
s.t.&\mathbf{a}^\top\left(\mathbf{b}+\mathbf{r}\right)\le 0~,\left\|\mathbf{r}\right\|^2 - \|\mathbf{b}\|^2 \le 0,~\left(\mathbf{c+r}\right)^\top\mathbf{y}=0\nonumber.
\end{eqnarray}
In the preceding equation, $\boldsymbol{\theta}=\mathbf{c+r}$, where $\mathbf{r}$ is the unknown, and $\hat{\mathbf{f}}$, $\mathbf{a}$, $\mathbf{b}$, $\mathbf{c}$, and $\mathbf{y}$ are known. Since the following equation holds:
\begin{equation}
\max|x|=\max\left\{-\min(x),\max(x)\right\}=\max\left\{-\min(x),-\min(-x)\right\},
\end{equation}
$\max{\left|\left(\mathbf{c+r}\right)^\top\hat{\mathbf{f}}\right|}$ can be decomposed to the following two sub-problems:
\begin{eqnarray}
&m_1=-\min\boldsymbol{\theta}_2^\top\mathbf{\hat{f}}=-\min{\mathbf{r}^\top\hat{\mathbf{f}}}-\mathbf{c}^\top\hat{\mathbf{f}}\label{eq:bound1}\\
s.t.&\mathbf{a}^\top\left(\mathbf{b}+\mathbf{r}\right)\le 0~,\left\|\mathbf{r}\right\|^2 - \|\mathbf{b}\|^2 \le 0,~\left(\mathbf{c+r}\right)^\top\mathbf{y}=0\nonumber,
\end{eqnarray}
\begin{eqnarray}
&m_2=\max\boldsymbol{\theta}^\top\mathbf{\hat{f}}=-\min\boldsymbol{\theta}_2^\top\left(-\mathbf{\hat{f}}\right)=- \min{\mathbf{r}^\top\left(-\hat{\mathbf{f}}\right)}-\mathbf{c}^\top\left(-\hat{\mathbf{f}}\right)\label{eq:bound2}\\
s.t.&\mathbf{a}^\top\left(\mathbf{b}+\mathbf{r}\right)\le 0~,\left\|\mathbf{r}\right\|^2 - \|\mathbf{b}\|^2 \le 0,~\left(\mathbf{c+r}\right)^\top\mathbf{y}=0\nonumber,
\end{eqnarray}
and 
\begin{equation}
\max\left|\boldsymbol{\theta}_2^\top\mathbf{\hat{f}}\right|=\max{\left|\left(\mathbf{c+r}\right)^\top\hat{\mathbf{f}}\right|}=\max\left(m_1,m_2\right).
\end{equation}

Therefore, our key is to solve the following problem:
\begin{eqnarray}
&\min{\mathbf{r}^\top\hat{\mathbf{f}}}\label{eq:bound_r}\\
s.t.&\mathbf{a}^\top\left(\mathbf{b}+\mathbf{r}\right)\le 0~,\left\|\mathbf{r}\right\|^2 - \|\mathbf{b}\|^2 \le 0,~\left(\mathbf{c+r}\right)^\top\mathbf{y}=0\nonumber.
\end{eqnarray}
Its Lagrangian multiplier can be written as:
\begin{equation}
L\left(\mathbf{r},\alpha,\beta,\rho\right) = \mathbf{r}^\top\hat{\mathbf{f}}+\alpha\mathbf{a^\top\left(b+r\right)}+\frac{1}{2}\beta\left(\|\mathbf{r}\|^2_2-\|\mathbf{b}\|^2_2\right)+\rho\left(\mathbf{c}+\mathbf{r}\right)^\top\mathbf{y}.\label{eq:lag-org}
\end{equation}
The corresponding Karush-Kuhn-Tucker (KKT) conditions are:
\begin{eqnarray}
\alpha &\ge& 0,~~~\mbox{(dual feasibility)}\label{eq:kkt-fist}\\
\beta &\ge& 0,\\
\|\mathbf{r}\|^2_2-\|\mathbf{b}\|^2_2&\le&0,~~~\mbox{(primal feasibility)}\label{eq:a_brm1}\\
\mathbf{a^\top\left(b+r\right)}&\le&0,\label{eq:a_br}\\
\left(\mathbf{c}+\mathbf{r}\right)^\top\mathbf{y}&=&0,\label{eq:a_brp1}\\
\alpha\mathbf{a^\top\left(b+r\right)}&=&0,~~~\mbox{(complementary slackness)}\\
\beta\left(\|\mathbf{r}\|^2_2-\|\mathbf{b}\|^2_2\right)&=&0,\\
\nabla_{\mathbf{r}}L\left(\mathbf{r},\alpha,\beta,\rho\right)&=& 0.~~~\mbox{(stationarity)}\label{eq:kkt-last}
\end{eqnarray}
Since the problem specified in Eq.~(\ref{eq:bound_r}) is lower bounded by $-\|\mathbf{b}\|_2\|\mathbf{f}\|_2$, it is clear that $\min\nolimits_{\mathbf{r}}L\left(\mathbf{r},\alpha,\beta,\rho\right)$ must also be bounded from below. In the following we study the four cases listed below:
\begin{itemize}
\item [1.] $\beta=0,~\mathbf{\hat{f}+\alpha a+\rho y}\neq 0$,
\item [2.] $\beta=0,~\mathbf{\hat{f}+\alpha a+\rho y}=0$,
\item [3.] $\beta>0,~\alpha = 0$,
\item [4.] $\beta>0,~\alpha > 0$.
\end{itemize}

\subsection{The Case: $\beta=0,~\mathbf{\hat{f}+\alpha a+\rho y}\neq \mathbf{0}$}
In this case, by setting $\mathbf{r}=t\left(\mathbf{f+\alpha a + \rho y}\right)$, and let $t\rightarrow -\infty$. We will have $L\left(\mathbf{r},\alpha,0,\rho\right)\rightarrow -\infty$. This is contradict to the observation that $\min\nolimits_{\mathbf{r}}L\left(\mathbf{r},\alpha,\beta,\rho\right)$ must be bounded from below. So when $\mathbf{\hat{f}+\alpha a+\rho y}\neq \mathbf{0}$, $\beta$ must be positive.

\subsection{The Case: $\beta=0,~\mathbf{\hat{f}+\alpha a+\rho y}= \mathbf{0}$}
Let $P_{\mathbf{u}}\left(\mathbf{v}\right)=\mathbf{v}-\frac{\mathbf{v}^\top\mathbf{u}}{\|\mathbf{u}\|^2_2}\mathbf{u}$ be the projection that project $\mathbf{v}$ to the null-space of $\mathbf{u}$. Given $\mathbf{\hat{f}+\alpha a+\rho y}= \mathbf{0}$, it is easy to verified that $\alpha P_{\mathbf{y}}\left(\mathbf{a}\right)=-P_{\mathbf{y}}\left(\mathbf{\hat{f}}\right)$. This suggests that $\alpha P_{\mathbf{y}}\left(\mathbf{a}\right)$ and $P_{\mathbf{y}}\left(\mathbf{\hat{f}}\right)$ are colinear. Also since $\alpha\ge 0$,  it must hold:
\begin{equation}
\frac{P_{\mathbf{y}}\left(\mathbf{a}\right)^\top P_{\mathbf{y}}\left(\mathbf{\hat{f}}\right)}{\|P_{\mathbf{y}}\left(\mathbf{a}\right)\|\|P_{\mathbf{y}}\left(\mathbf{\hat{f}}\right)\|}=-1.\label{eq:b0cond}
\end{equation}
Given $\alpha P_{\mathbf{y}}\left(\mathbf{a}\right)=-P_{\mathbf{y}}\left(\mathbf{\hat{f}}\right)$, $\alpha$ can be computed by:
\begin{equation}
\alpha=-\frac{P_{\mathbf{y}}\left(\mathbf{a}\right)^\top P_{\mathbf{y}}\left(\mathbf{\hat{f}}\right)}{\|P_{\mathbf{y}}\left(\mathbf{a}\right)\|_2^2}=\frac{\|P_{\mathbf{y}}\left(\mathbf{\hat{f}}\right)\|_2}{\|P_{\mathbf{y}}\left(\mathbf{a}\right)\|_2}.
\end{equation}
Similarly, the value of $\rho$ can be computed by:
\begin{equation}
\rho=-\frac{\mathbf{\hat{f}^\top y}}{\|\mathbf{y}\|^2_2}-\alpha\frac{\mathbf{a^\top y}}{\|\mathbf{y}\|^2_2}=-\frac{\mathbf{\hat{f}^\top y}}{\|\mathbf{y}\|^2_2}-\frac{\|P_{\mathbf{y}}\left(\mathbf{\hat{f}}\right)\|_2}{\|P_{\mathbf{y}}\left(\mathbf{a}\right)\|_2}\frac{\mathbf{a^\top \mathbf{y}}}{\|\mathbf{y}\|^2_2}
\end{equation}
By plugging $\beta=0$ and the obtained value of $\alpha$ and $\rho$ into Eq.~(\ref{eq:lag-org}), it follows:
\begin{eqnarray}
L\left(\mathbf{r},\alpha,0,\rho\right) &=& \alpha\mathbf{a^\top b}+\mathbf{c^\top \left(\rho y\right)}\nonumber\\
&=& \alpha\mathbf{a^\top b}+\mathbf{c^\top \left(-\hat{f}-\alpha a\right)}\nonumber\\
&=& \alpha\mathbf{a^\top \left(b-c\right)}-\mathbf{c^\top \hat{f}}\nonumber\\
&=&-\frac{\|P_{\mathbf{y}}\left(\mathbf{\hat{f}}\right)\|_2}{\|P_{\mathbf{y}}\left(\mathbf{a}\right)\|_2}\mathbf{a^\top \boldsymbol{\theta_1}}-\mathbf{c^\top \hat{f}}\label{eq:b0min1}
\end{eqnarray}

It can be verified that in this case, all the KKT conditions specified in Eq.~(\ref{eq:kkt-fist})-Eq.~(\ref{eq:kkt-last}) are all satisfied. Since the problem defined in Eq.~(\ref{eq:bound1}) is convex with a convex domain, Eq.~(\ref{eq:b0min1}) defines its minimum. 

The following theorem summarize the result for the case $\beta=0$.
\begin{theorem}~\label{th:b0}
When $\frac{P_{\mathbf{y}}\left(\mathbf{a}\right)^\top P_{\mathbf{y}}\left(\mathbf{\hat{f}}\right)}{\|P_{\mathbf{y}}\left(\mathbf{a}\right)\|\|P_{\mathbf{y}}\left(\mathbf{\hat{f}}\right)\|}=-1$, $\mathbf{r}^\top\hat{\mathbf{f}}$ achieves its minimum value at $\beta=0$, and this minimum value can be computed as:
\begin{equation}
\min\limits_{\mathbf{r}}\mathbf{r}^\top\hat{\mathbf{f}}=-\frac{\|P_{\mathbf{y}}\left(\mathbf{\hat{f}}\right)\|_2}{\|P_{\mathbf{y}}\left(\mathbf{a}\right)\|_2}\mathbf{a^\top \boldsymbol{\theta_1}}-\mathbf{c^\top \hat{f}}.\label{eq:b0min}
\end{equation}
And in this case, we have:
\begin{equation}
\alpha=\frac{\|P_{\mathbf{y}}\left(\mathbf{\hat{f}}\right)\|_2}{\|P_{\mathbf{y}}\left(\mathbf{a}\right)\|_2},~\beta=0,~\rho=-\frac{\mathbf{\hat{f}^\top y}}{\|\mathbf{y}\|^2_2}-\frac{\|P_{\mathbf{y}}\left(\mathbf{\hat{f}}\right)\|_2}{\|P_{\mathbf{y}}\left(\mathbf{a}\right)\|_2}\frac{\mathbf{a^\top \mathbf{y}}}{\|\mathbf{y}\|^2_2}.
\end{equation}
\end{theorem}
In this case, since $\alpha=\frac{\|P_{\mathbf{y}}\left(\mathbf{\hat{f}}\right)\|_2}{\|P_{\mathbf{y}}\left(\mathbf{a}\right)\|_2}>0$, the minimum value is achieved on the hyperplane defined by $\mathbf{a^\top\left(b+r\right)}=0$. To compute Eq.~(\ref{eq:b0cond}) and Eq.~(\ref{eq:b0min}), $\|\mathbf{P_{\mathbf{y}}\left(\hat{f}\right)}\|_2$, $\mathbf{\hat{f}}^\top \mathbf{y}$, $\mathbf{\hat{f}}^\top \mathbf{1}$, $\mathbf{y}^\top \mathbf{y}$, and $\mathbf{y}^\top \mathbf{1}$ are independent to $\lambda_1$, $\lambda_2$ and $\boldsymbol{\theta_1}$, therefore, can be precomputed. $\|\mathbf{P_{\mathbf{y}}\left(a\right)}\|_2$ and $\mathbf{a^\top \boldsymbol{\theta_1}}$ can be shared by all features. These properties can be used to accelerate the computation of the screening rule. For each feature, the only expensive computation is $\mathbf{\hat{f}}^\top \boldsymbol{\theta}_1$, and it can be accelerated by utilizing the sparse structure of $ \boldsymbol{\theta}_1$.

\begin{corollary}\label{cor:b0}
When $\frac{\left|P_{\mathbf{y}}\left(\mathbf{a}\right)^\top P_{\mathbf{y}}\left(\mathbf{\hat{f}}\right)\right|}{\|P_{\mathbf{y}}\left(\mathbf{a}\right)\|\|P_{\mathbf{y}}\left(\mathbf{\hat{f}}\right)\|}=1$, $\mathbf{r}^\top\hat{\mathbf{f}}$ achieve its maximum value at $\beta=0$, and in this case $-\min\boldsymbol{\theta}^\top\mathbf{\hat{f}}$ can be computed as:
\begin{equation}
-\min\boldsymbol{\theta}_2^\top\mathbf{\hat{f}}=-\min{\mathbf{r}^\top\hat{\mathbf{f}}}-\mathbf{c}^\top\hat{\mathbf{f}}=\frac{\|P_{\mathbf{y}}\left(\mathbf{\hat{f}}\right)\|_2}{\|P_{\mathbf{y}}\left(\mathbf{a}\right)\|_2}\mathbf{a^\top \boldsymbol{\theta_1}}.\label{eq:b0min_all}
\end{equation}
\end{corollary}

\subsection{The Case: $\beta>0,~\alpha = 0$}
In this case, since $\beta>0$ and $\alpha = 0$, the minimum value of $\mathbf{r}^\top\hat{\mathbf{f}}$ is achieved on the boundary of the hyperball. In Figure~\ref{fig:BND-1}, it corresponds to the arc of the red circle under the blue line. By plugging $\alpha=0$ in Eq.~(\ref{eq:lag-org}), it can be obtained:
\begin{equation}
L\left(\mathbf{r},0,\beta,\rho\right) = \mathbf{r}^\top\hat{\mathbf{f}}+\frac{1}{2}\beta\left(\|\mathbf{r}\|^2_2-\|\mathbf{b}\|^2_2\right)+\rho\left(\mathbf{c}+\mathbf{r}\right)^\top\mathbf{y}
\end{equation}
The dual function $g\left(0,\beta,\rho\right)=\min\nolimits_{\mathbf{r}}L\left(\mathbf{r},0,\beta,\rho\right)$ can be obtained by setting
\begin{equation}\nabla_\mathbf{r}L\left(\mathbf{r},0,\beta,\rho\right)=\mathbf{\hat{f}+\beta r+ \rho y}=0\Rightarrow \mathbf{r}=-\frac{1}{\beta}\left(\mathbf{\hat{f}+\rho y}\right).
\end{equation}
Since $\beta>0$, it must hold that $\|\mathbf{b}\|_2=\|\mathbf{r}\|_2$. Therefore $\beta$ can be written as:
\begin{equation}
\beta=\frac{\|\mathbf{\hat{f} + \rho y}\|_2}{\|\mathbf{b}\|_2}
\end{equation}
Plugging the obtained $\mathbf{r}$ and $\beta$ into $L\left(\mathbf{r},0,\beta,\rho\right)$ leads to the following result:
\begin{equation}
g\left(\rho\right)=\min\limits_{\mathbf{r}}L\left(\mathbf{r},0,\beta,\rho\right)=-\|\mathbf{b}\|_2\|\mathbf{\hat{f}+\rho y}\|_2+\rho\mathbf{c^\top y}.\label{eq:a0dual}
\end{equation}
To maximize the dual function, we simply set $\frac{\partial g\left(\rho\right)}{\partial \rho}=0$. Also by noticing that $\mathbf{b^\top y}=\mathbf{c^\top y}$, as $\boldsymbol{\theta}_1^\top \mathbf{y}=0$, the following equation can be obtained:
\begin{equation}
-\|\mathbf{b}\|_2\frac{\mathbf{\rho y^\top y + \hat{f}^\top y}}{\|\mathbf{\hat{f}+\rho y}\|_2} + \mathbf{b}^\top\mathbf{y} = 0.\label{eq:a0L1}
\end{equation}

Taking square on both sides of the equation and simplifying it, we have:
\begin{eqnarray}
0&=&\rho^2\mathbf{y^\top y}\left(\mathbf{b^\top b~y^\top y-\left(b^\top y\right)^2}\right)\\
&-& 2\rho\mathbf{\hat{f}^\top y}\left(-\mathbf{b^\top b~y^\top y+\left(b^\top y\right)^2}\right)\nonumber\\
&+& \mathbf{b^\top b\left(\hat{f}^\top y\right)^2-\hat{f}^\top \hat{f}\left(\mathbf{b}^\top y\right)^2} \nonumber.
\end{eqnarray}
Solving the preceding equation leads to the result:
\begin{equation}
\rho=-\frac{\mathbf{\hat{f}^\top y}}{\mathbf{y ^\top y}}\pm \frac{\|P_\mathbf{y}\left(\mathbf{\hat{f}}\right)\|_2}{\|P_\mathbf{y}\left(\mathbf{b}\right)\|_2}\frac{\mathbf{b^\top y}}{\mathbf{y ^\top y}}.\label{eq:a0r}
\end{equation}
To obtain this equation, we used the fact:
\begin{eqnarray}
\mathbf{b^\top b}-\frac{\left(\mathbf{b^\top y}\right)^2}{\mathbf{y^\top y}}=\left\|\mathbf{b-\frac{b^\top y}{y^\top y}y}\right\|^2_2=\left\|P_{\mathbf{y}}\left(\mathbf{b}\right)\right\|_2^2,\\
\mathbf{\hat{f}^\top \hat{f}}-\frac{\left(\mathbf{\hat{f}^\top y}\right)^2}{\mathbf{y^\top y}}=\left\|\mathbf{\hat{f}-\frac{\hat{f}^\top y}{y^\top y}y}\right\|^2_2=\left\|P_{\mathbf{y}}\left(\mathbf{\hat{f}}\right)\right\|_2^2~.
\end{eqnarray}

Since $\left(\mathbf{c}+\mathbf{r}\right)^\top\mathbf{y}=0$ and $\mathbf{r}=-\frac{1}{\beta}\left(\mathbf{\hat{f}+\rho y}\right)$, we have  $\beta=\frac{\mathbf{\hat{f}^\top y +\rho y^\top y}}{\mathbf{c^\top y}}$. To ensure that $\beta$ is positive, we must have:
\begin{equation}
\rho=-\frac{\mathbf{\hat{f}^\top y}}{\mathbf{y ^\top y}}+ \frac{\|P_\mathbf{y}\left(\mathbf{\hat{f}}\right)\|_2}{\|P_\mathbf{y}\left(\mathbf{b}\right)\|_2}\frac{\mathbf{b^\top y}}{\mathbf{y ^\top y}}.\label{eq:a0r1}
\end{equation}
And in this case, $\beta$ can be written in the form:
\begin{equation}
\beta=\frac{\|\mathbf{\hat{f} + \rho y}\|_2}{\|\mathbf{b}\|_2}=\frac{\|P_\mathbf{y}\left(\mathbf{\hat{f}}\right)\|_2}{\|P_\mathbf{y}\left(\mathbf{b}\right)\|_2}
\end{equation}

To compute $\max\limits_{\rho}g\left(\rho\right)$, first, we notice that Eq.~(\ref{eq:a0L1}) can be rewritten as:
\begin{equation}
-\|\mathbf{b}\|_2\frac{\mathbf{\rho y^\top y + \hat{f}^\top y}}{\|\mathbf{\hat{f}+\rho y}\|_2} + \mathbf{b}^\top\mathbf{y} = 0 \Rightarrow \|\mathbf{b}\|_2 \|\mathbf{\hat{f}+\rho y}\|_2=\|\mathbf{b}\|_2^2\frac{\mathbf{\rho y^\top y + \hat{f}^\top y}}{\mathbf{b}^\top\mathbf{y}}.\label{eq:a0L2}
\end{equation}
By plugging Eq.~(\ref{eq:a0r1}) and Eq.~(\ref{eq:a0L2}) into Eq.~(\ref{eq:a0dual}) we have:
\begin{eqnarray}
\max\limits_{\rho}g\left(\rho\right)&=&-\|\mathbf{b}\|_2^2\frac{\mathbf{\rho y^\top y + \hat{f}^\top y}}{\mathbf{b}^\top\mathbf{y}}+\rho\mathbf{b^\top y}\nonumber\\
&=&- \left\|P_{\mathbf{y}}\left(\mathbf{b}\right)\right\|_2 \left\|P_{\mathbf{y}}\left(\mathbf{\hat{f}}\right)\right\|_2-\frac{\mathbf{\hat{f}^\top y b^\top y}}{\mathbf{y^\top y}}
\end{eqnarray}
Since $\frac{\mathbf{\hat{f}^\top y b^\top y}}{\mathbf{y^\top y}}=\mathbf{\hat{f}^\top b}-P_{\mathbf{y}}^\top\left(\mathbf{b}\right)P_{\mathbf{y}}\left(\mathbf{\hat{f}}\right)$, $\max\limits_{\rho}g\left(\rho\right)$ can also be written as:
\begin{equation}
\max\limits_{\rho}g\left(\rho\right)=-\left\|P_{\mathbf{y}}\left(\mathbf{b}\right)\right\|_2 \left\|P_{\mathbf{y}}\left(\mathbf{\hat{f}}\right)\right\|_2+P_{\mathbf{y}}\left(\mathbf{b}\right)^\top P_{\mathbf{y}}\left(\mathbf{\hat{f}}\right)-\mathbf{\hat{f}^\top b}.\nonumber
\end{equation}

It can be verified that in this case, all the KKT conditions specified in Eq.~(\ref{eq:kkt-fist})-Eq.~(\ref{eq:a_brm1}) and Eq.~(\ref{eq:a_brp1})-Eq.(\ref{eq:kkt-last})  are satisfied. We still need to study that under which condition Eq.~(\ref{eq:a_br}) can be satisfied.
By setting the derivative of Eq.~(\ref{eq:lag-org}) to be zero, the following equation can be obtained:
$$
\mathbf{r}=-\frac{1}{\beta}\left(\mathbf{\hat{f}+\alpha a + \rho y}\right)
$$
Plugging this equation to $\mathbf{a^\top\left(b+r\right)}\le0$, we have:
\begin{equation}
\alpha \ge \mathbf{\beta a^\top b -a^\top f- \rho a^\top y}.
\end{equation}
If $\mathbf{\beta a^\top b -a^\top f- \rho a^\top y} >0$, we must have $\alpha>0$, according to complementary slackness condition, we have $\mathbf{a^\top\left(b+r\right)} = 0$. Therefore $\alpha = \mathbf{\beta a^\top b -a^\top f- \rho a^\top y}$. On the other hand, if $\mathbf{\beta a^\top b -a^\top f- \rho a^\top y} \le 0$, we must have $\alpha =0$. Since, if $\alpha >0$, we will have $\alpha = \mathbf{\beta a^\top b -a^\top f- \rho a^\top y}\le 0$, which forms a contradiction. Therefore, to ensure that Eq.~(\ref{eq:a_br}) is satisfied, we need to have $\mathbf{\beta a^\top b -a^\top f- \rho a^\top y} \le 0$. By plugging the obtained $\beta$ and $\rho$, we have:
\begin{equation}
\mathbf{\beta a^\top b -a^\top f- \rho a^\top y}=\left\|P_{\mathbf{y}}\left(\mathbf{\hat{f}}\right)\right\|_2P_{\mathbf{y}}\left(\mathbf{a}\right)^\top\left(\frac{P_{\mathbf{y}}\left(\mathbf{b}\right)}{\left\|P_{\mathbf{y}}\left(\mathbf{b}\right)\right\|_2}-\frac{P_{\mathbf{y}}\left(\mathbf{\hat{f}}\right)}{\left\|P_{\mathbf{y}}\left(\mathbf{\hat{f}}\right)\right\|_2}\right)
\end{equation}
Therefore, if $P_{\mathbf{y}}\left(\mathbf{a}\right)^\top\left(\frac{P_{\mathbf{y}}\left(\mathbf{b}\right)}{\left\|P_{\mathbf{y}}\left(\mathbf{b}\right)\right\|_2}-\frac{P_{\mathbf{y}}\left(\mathbf{\hat{f}}\right)}{\left\|P_{\mathbf{y}}\left(\mathbf{\hat{f}}\right)\right\|_2}\right)\le 0$, we must have $\alpha=0$. And in this case, the KKT condition $\mathbf{a^\top\left(b+r\right)} \ge 0$ is also satisfied. 

The following theorem summarize the result for the case $\beta>0,~\alpha = 0$.
\begin{theorem}
When $P_{\mathbf{y}}\left(\mathbf{a}\right)^\top\left(\frac{P_{\mathbf{y}}\left(\mathbf{b}\right)}{\left\|P_{\mathbf{y}}\left(\mathbf{b}\right)\right\|_2}-\frac{P_{\mathbf{y}}\left(\mathbf{\hat{f}}\right)}{\left\|P_{\mathbf{y}}\left(\mathbf{\hat{f}}\right)\right\|_2}\right)\le 0$, $\mathbf{r}^\top\hat{\mathbf{f}}$ achieves its minimum value at $\beta>0$ and $\alpha = 0$:
\begin{equation}
\min\limits_{\mathbf{r}}\mathbf{r}^\top\hat{\mathbf{f}}= -\left\|P_{\mathbf{y}}\left(\mathbf{b}\right)\right\|_2 \left\|P_{\mathbf{y}}\left(\mathbf{\hat{f}}\right)\right\|_2+P_{\mathbf{y}}\left(\mathbf{b}\right)^\top P_{\mathbf{y}}\left(\mathbf{\hat{f}}\right)-\mathbf{\hat{f}^\top b}\label{eq:a0min}
\end{equation}
In this case, we have:
\begin{equation}
\alpha=0,~\beta=\frac{\|P_\mathbf{y}\left(\mathbf{\hat{f}}\right)\|_2}{\|P_\mathbf{y}\left(\mathbf{b}\right)\|_2},~\rho=-\frac{\mathbf{\hat{f}^\top y}}{\mathbf{y ^\top y}}- \frac{\|P_\mathbf{y}\left(\mathbf{\hat{f}}\right)\|_2}{\|P_\mathbf{y}\left(\mathbf{b}\right)\|_2}\frac{\mathbf{b^\top y}}{\mathbf{y ^\top y}}.\label{eq:a0minp}
\end{equation}
\end{theorem}

Note that in Eq.~(\ref{eq:a0min}) and Eq.~(\ref{eq:a0minp}), $\mathbf{\hat{f}^\top \hat{f}}$, $\mathbf{\hat{f}^\top y}$, $\mathbf{y^\top y}$, and $\left\|P_{\mathbf{y}}\left(\mathbf{\hat{f}}\right)\right\|_2$ does not rely on $\lambda_2$ and $\boldsymbol{\theta}_1$, therefore, can be precomputed. $\mathbf{b^\top b}$, $\mathbf{b^\top y}$ and $\left\|P_{\mathbf{y}}\left(\mathbf{b}\right)\right\|_2$, although relying on $\lambda_2$ or $\boldsymbol{\theta}_1$, are shared by all features. These properties can be used to accelerate computation when implementing the screening rule.

\begin{corollary}
When $P_{\mathbf{y}}\left(\mathbf{a}\right)^\top\left(\frac{P_{\mathbf{y}}\left(\mathbf{b}\right)}{\left\|P_{\mathbf{y}}\left(\mathbf{b}\right)\right\|_2}-\frac{P_{\mathbf{y}}\left(\mathbf{\hat{f}}\right)}{\left\|P_{\mathbf{y}}\left(\mathbf{\hat{f}}\right)\right\|_2}\right)\le 0$, $\mathbf{r}^\top\hat{\mathbf{f}}$ achieves its minimum value at $\beta>0$ and $\alpha = 0$. And $-\min\boldsymbol{\theta}^\top\mathbf{\hat{f}}$ can be computed as:
\begin{eqnarray}
-\min\boldsymbol{\theta}_2^\top\mathbf{\hat{f}}&=&-\min{\mathbf{r}^\top\hat{\mathbf{f}}}-\mathbf{c}^\top\hat{\mathbf{f}}\nonumber\\
&=& \left\|P_{\mathbf{y}}\left(\mathbf{b}\right)\right\|_2 \left\|P_{\mathbf{y}}\left(\mathbf{\hat{f}}\right)\right\|_2-P_{\mathbf{y}}\left(\mathbf{b}\right)^\top P_{\mathbf{y}}\left(\mathbf{\hat{f}}\right)-\mathbf{\hat{f}^\top \boldsymbol{\theta}_1}\label{eq:a0min_all}
\end{eqnarray}
\end{corollary}

\subsection{The Case: $\beta>0,~\alpha > 0$\label{sec:bn0an0}}
In this case, the minimum value of $\mathbf{r}^\top\hat{\mathbf{f}}$ is achieved on the intersection of the boundary of the hyperball and the hyperplane. In Figure~\ref{fig:BND-1}, this corresponds to the two red points on the intersection of the red circle and the blue line. It turns out that, in the case $\beta>0,~\alpha > 0$, deriving a closed form solution for the problem specified in Eq.~(\ref{eq:bound1}) is not easy. Theorem~\ref{th:ball-same} suggests that when the minimum value is achieved on the intersection of the hyperball and the hyperplane, we could switch the hyperball used in Eq.~(\ref{eq:bound1}) to simplify the computation. Below, we show that a closed form solution can be obtained by using the hyperball $\mathbf{B}_t$ with $t=1-\left(\frac{1}{\lambda_2}-\frac{1}{\lambda_1}\right)\mathbf{a}^\top\mathbf{1}$. This corresponds to the hyperball defined in Theorem~\ref{th:small-ball}. As proved in Theorem~\ref{th:ball-same}, the intersections of different $\mathbf{B}_t$ and $\left(\boldsymbol{\theta}_1-\frac{\mathbf{1}}{\lambda_1}\right)^\top\left(\boldsymbol{\theta}_2-\boldsymbol{\theta}_1\right)= 0$ are identical. Therefore, switching the hyperball $\mathbf{B}_t$ in this case does not change the maximum value of $|\boldsymbol{\theta}^\top\hat{\mathbf{f}}|$.

When $\mathbf{B}_t$ with $t=1-\left(\frac{1}{\lambda_2}-\frac{1}{\lambda_1}\right)\mathbf{a}^\top\mathbf{1}$ is used and assume that the minimum is achieved on the boundary of the hyperball and the hyperplane, the problem specified in Eq.~(\ref{eq:bound1}) can be rewritten as:
\begin{eqnarray}
&\arg\limits_{\mathbf{r}}\min{\mathbf{r}^\top\hat{\mathbf{f}}}\label{eq:bn0an0bnd}\\
s.t.&\mathbf{a}^\top\mathbf{r}=0,~\left\|\mathbf{r}\right\|_2^2 -l^2 \le 0,~\left(\mathbf{\hat{c}+r}\right)^\top\mathbf{y}=0\nonumber.
\end{eqnarray}

And its Lagrangian multiplier can be written as:
\begin{equation}
L\left(\mathbf{r},\alpha,\beta,\rho\right) = \mathbf{r}^\top\hat{\mathbf{f}}+\alpha\mathbf{a^\top r}+\frac{1}{2}\beta\left(\|\mathbf{r}\|^2_2-l^2\right)+\rho\left(\mathbf{\hat{c}}+\mathbf{r}\right)^\top\mathbf{y},\label{eq:lag-be0}
\end{equation}
In the preceding equation, $\mathbf{c}$ is center of the hyperfall, and $l$ is the radius of the hyperfall, which are defined as: $$\mathbf{\hat{c}}=\frac{1}{2}\left(\frac{1}{\lambda_2}-\frac{1}{\lambda_1}\right)P_{\mathbf{a}}\left(\mathbf{1}\right)+\boldsymbol{\theta}_1,~~l=\frac{1}{2}\left(\frac{1}{\lambda_2}-\frac{1}{\lambda_1}\right)\left\|P_{\mathbf{a}}\left(\mathbf{1}\right)\right\|.$$

The dual function $g\left(\alpha,\beta,\rho\right)=\min\nolimits_{\mathbf{r}}L\left(\mathbf{r},\alpha,\beta,\rho\right)$ can be obtained by setting
\begin{equation}\nabla_\mathbf{r}L\left(\mathbf{r},\alpha,\beta,\rho\right)=\mathbf{\hat{f}+\alpha a + \beta r+ \rho y}=0\Rightarrow \mathbf{r}=-\frac{1}{\beta}\left(\mathbf{\hat{f}+\alpha a+\rho y}\right).
\end{equation}
Since $\beta\neq 0$, it must hold that $\|\mathbf{r}\|_2=l$. Therefore $\beta$ can be written as:
\begin{equation}
\beta=\frac{\|\mathbf{\hat{f} + \alpha a+ \rho y}\|_2}{l}
\end{equation}
Since $\alpha\neq 0$, it must hold that $\mathbf{a}^\top \mathbf{r}=0$. Therefore $\alpha$ can be written as:
\begin{equation}
\alpha=\mathbf{-a^\top\left(\hat{f}+\rho y\right)}
\end{equation}
Plugging the obtained $\mathbf{r}$, $\alpha$ and $\beta$ into $L\left(\mathbf{r},\alpha,\beta,\rho\right)$ leads to the following result:
\begin{eqnarray}
g\left(\rho\right)=\min\limits_{\mathbf{r}}L\left(\mathbf{r},\alpha,\beta,\rho\right)&=&-l\|\mathbf{\hat{f}+ \alpha a + \rho y}\|_2+\rho\mathbf{\hat{c}^\top y}\nonumber\\
&=&-l\|\mathbf{\hat{f} - a^\top\hat{f} a+ \rho y-a^\top y a}\|_2+\rho\mathbf{\hat{c}^\top y}\nonumber\\
&=&-l\|P_{\mathbf{a}}\left(\mathbf{\hat{f}}\right) + \rho P_{\mathbf{a}}\left(\mathbf{y}\right)\|_2+\rho\mathbf{\hat{c}^\top y}.\label{eq:an0dual}
\end{eqnarray}
To maximize $g\left(\rho\right)$, we simply set $\frac{\partial g\left(\rho\right)}{\partial \rho}=0$, which leads to the equation:
\begin{equation}
l\frac{\rho P_{\mathbf{a}}\left(\mathbf{y}\right)^\top P_{\mathbf{a}}\left(\mathbf{y}\right) + P_{\mathbf{a}}\left(\mathbf{\hat{f}}\right)^\top P_{\mathbf{a}}\left(\mathbf{y}\right)}{\|P_{\mathbf{a}}\left(\mathbf{\hat{f}}\right) + \rho P_{\mathbf{a}}\left(\mathbf{y}\right)\|_2} = \mathbf{\hat{c}}^\top\mathbf{y}.\label{eq:an0L1}
\end{equation}
Take square on both sides of the equation and do some simplification. The following equation can be obtained:
\begin{eqnarray}
0&=&\rho^2P_{\mathbf{a}}\left(\mathbf{y}\right)^\top P_{\mathbf{a}}\left(\mathbf{y}\right)\left(\left(P_{\mathbf{a}}\left(\mathbf{1}\right)^\top P_{\mathbf{a}}\left(\mathbf{y}\right)\right)^2-P_{\mathbf{a}}\left(\mathbf{1}\right)^\top P_{\mathbf{a}}\left(\mathbf{1}\right)P_{\mathbf{a}}\left(\mathbf{y}\right)^\top P_{\mathbf{a}}\left(\mathbf{y}\right)\right)\nonumber\\
&-& 2\rho P_{\mathbf{a}}\left(\mathbf{\hat{f}}\right)^\top P_{\mathbf{a}}\left(\mathbf{y}\right)\left(P_{\mathbf{a}}\left(\mathbf{1}\right)^\top P_{\mathbf{a}}\left(\mathbf{1}\right)P_{\mathbf{a}}\left(\mathbf{y}\right)^\top P_{\mathbf{a}}\left(\mathbf{y}\right)-\left(P_{\mathbf{a}}\left(\mathbf{1}\right)^\top P_{\mathbf{a}}\left(\mathbf{y}\right)\right)^2\right)\nonumber\\
&+& \left(P_{\mathbf{a}}\left(\mathbf{1}\right)^\top P_{\mathbf{a}}\left(\mathbf{y}\right)\right)^2 P_{\mathbf{a}}\left(\mathbf{\hat{f}}\right)^\top P_{\mathbf{a}}\left(\mathbf{\hat{f}}\right)-\left(P_{\mathbf{a}}\left(\mathbf{\hat{f}}\right)^\top P_{\mathbf{a}}\left(\mathbf{y}\right)\right)^2 P_{\mathbf{a}}\left(\mathbf{1}\right)^\top P_{\mathbf{a}}\left(\mathbf{1}\right)\nonumber.
\end{eqnarray}

To obtain the preceding equation, we used that fact that 
$$\mathbf{\hat{c}^\top y}=\frac{1}{2}\left(\frac{1}{\lambda_2}-\frac{1}{\lambda_1}\right)P_{\mathbf{a}}\left(\mathbf{y}\right)^\top P_{\mathbf{a}}\left(\mathbf{1}\right)~\mbox{and}~l^2=\frac{1}{4}\left(\frac{1}{\lambda_2}-\frac{1}{\lambda_1}\right)^2 P_{\mathbf{a}}\left(\mathbf{1}\right)^\top P_{\mathbf{a}}\left(\mathbf{1}\right).$$
Solving the problem results a closed form solution for $\rho$ in the following form:
\begin{equation}
\rho=-\frac{P_{\mathbf{a}}\left(\mathbf{\hat{f}}\right)^\top P_{\mathbf{a}}\left(\mathbf{y}\right)}{P_{\mathbf{a}}\left(\mathbf{y}\right)^\top P_{\mathbf{a}}\left(\mathbf{y}\right)}\pm \frac{\|P_{P_{\mathbf{a}}\left(\mathbf{y}\right)}\left(P_{\mathbf{a}}\left(\mathbf{\hat{f}}\right)\right)\|_2}{\|P_{P_{\mathbf{a}}\left(\mathbf{y}\right)}\left(P_{\mathbf{a}}\left(\mathbf{1}\right)\right)\|_2}\frac{P_{\mathbf{a}}\left(\mathbf{1}\right)^\top P_{\mathbf{a}}\left(\mathbf{y}\right)}{P_{\mathbf{a}}\left(\mathbf{y}\right)^\top P_{\mathbf{a}}\left(\mathbf{y}\right)}\label{eq:an0r}
\end{equation}
Since $\left(\mathbf{\hat{c}}+\mathbf{r}\right)^\top\mathbf{y}=0$, we have $\bigg(P_{\mathbf{a}}\left(\mathbf{\hat{c}}\right)+P_{\mathbf{a}}\left(\mathbf{r}\right)\bigg)^\top P_{\mathbf{a}}\left(\mathbf{y}\right)=0$. It can be verified that $P_{\mathbf{a}}\left(\mathbf{r}\right)=-\frac{1}{\beta}\left(P_{\mathbf{a}}\left(\mathbf{\hat{f}}\right)+\rho P_{\mathbf{a}}\left(\mathbf{y}\right)\right)$, we have  $\beta=\frac{P_{\mathbf{a}}\left(\mathbf{\hat{f}}\right)^\top P_{\mathbf{a}}\left(\mathbf{y}\right) +\rho P_{\mathbf{a}}\left(\mathbf{y}\right)^\top P_{\mathbf{a}}\left(\mathbf{y}\right)}{P_{\mathbf{a}}\left(\mathbf{\hat{c}}\right)^\top P_{\mathbf{a}}\left(\mathbf{y}\right)}$. To ensure that $\beta$ is positive, we must have:
\begin{equation}
\rho=-\frac{P_{\mathbf{a}}\left(\mathbf{\hat{f}}\right)^\top P_{\mathbf{a}}\left(\mathbf{y}\right)}{P_{\mathbf{a}}\left(\mathbf{y}\right)^\top P_{\mathbf{a}}\left(\mathbf{y}\right)}- \frac{\|P_{P_{\mathbf{a}}\left(\mathbf{y}\right)}\left(P_{\mathbf{a}}\left(\mathbf{\hat{f}}\right)\right)\|_2}{\|P_{P_{\mathbf{a}}\left(\mathbf{y}\right)}\left(P_{\mathbf{a}}\left(\mathbf{1}\right)\right)\|_2}\frac{P_{\mathbf{a}}\left(\mathbf{1}\right)^\top P_{\mathbf{a}}\left(\mathbf{y}\right)}{P_{\mathbf{a}}\left(\mathbf{y}\right)^\top P_{\mathbf{a}}\left(\mathbf{y}\right)}\label{eq:an0r1}
\end{equation}
And in this case, $\beta$ can be written in the form:
\begin{equation}
\beta=\frac{\|P_{\mathbf{a}}\left(\mathbf{\hat{f}}\right)+\rho P_{\mathbf{a}}\left(\mathbf{y}\right)\|_2}{l}=2\left(\frac{1}{\lambda_2}-\frac{1}{\lambda_1}\right)^{-1}\frac{\|P_{P_{\mathbf{a}}\left(\mathbf{y}\right)}\left(P_{\mathbf{a}}\left(\mathbf{\hat{f}}\right)\right)\|_2}{\|P_{P_{\mathbf{a}}\left(\mathbf{y}\right)}\left(P_{\mathbf{a}}\left(\mathbf{1}\right)\right)\|_2}
\end{equation}

To compute $\max\limits_{\rho}g\left(\rho\right)$, first, we notice that Eq.~(\ref{eq:an0L1}) can be rewritten as:
\begin{eqnarray}
&& \mathbf{\hat{c}}^\top\mathbf{y}=l\frac{\rho P_{\mathbf{a}}\left(\mathbf{y}\right)^\top P_{\mathbf{a}}\left(\mathbf{y}\right) + P_{\mathbf{a}}\left(\mathbf{\hat{f}}\right)^\top P_{\mathbf{a}}\left(\mathbf{y}\right)}{\|P_{\mathbf{a}}\left(\mathbf{\hat{f}}\right) + \rho P_{\mathbf{a}}\left(\mathbf{y}\right)\|_2} \nonumber\\
&\Rightarrow&l\|P_{\mathbf{a}}\left(\mathbf{\hat{f}}\right) + \rho P_{\mathbf{a}}\left(\mathbf{y}\right)\|_2=l^2\frac{\rho P_{\mathbf{a}}\left(\mathbf{y}\right)^\top P_{\mathbf{a}}\left(\mathbf{y}\right) + P_{\mathbf{a}}\left(\mathbf{\hat{f}}\right)^\top P_{\mathbf{a}}\left(\mathbf{y}\right)}{\mathbf{\hat{c}}^\top\mathbf{y}}.\label{eq:an0L2}
\end{eqnarray}
By plugging Eq.~(\ref{eq:an0r1}) and Eq.~(\ref{eq:an0L2}) into Eq.~(\ref{eq:an0dual}) we have:
\begin{eqnarray}
\max\limits_{\rho}g\left(\rho\right)=\frac{1}{2}\left(\frac{1}{\lambda_2}-\frac{1}{\lambda_1}\right)\Bigg(-\left\|P_{P_{\mathbf{a}}\left(\mathbf{y}\right)}\left(P_{\mathbf{a}}\left(\mathbf{\hat{f}}\right)\right)\right\|_2 \left\|P_{P_{\mathbf{a}}\left(\mathbf{y}\right)}\left(P_{\mathbf{a}}\left(\mathbf{1}\right)\right)\right\|_2\nonumber\\
-\frac{P^\top_{\mathbf{a}}\left(\mathbf{\hat{f}}\right)^\top P_{\mathbf{a}}\left(\mathbf{y}\right)P_{\mathbf{a}}\left(\mathbf{1}\right)^\top P_{\mathbf{a}}\left(\mathbf{y}\right)}{P_{\mathbf{a}}\left(\mathbf{y}\right)^\top P_{\mathbf{a}}\left(\mathbf{y}\right)} \Bigg).\label{eq:an0min1}
\end{eqnarray}
Since $P_{\mathbf{a}}\left(\mathbf{1}\right)^\top P_{\mathbf{a}}\left(\mathbf{f}\right)-\frac{P_{\mathbf{a}}\left(\mathbf{\hat{f}}\right)^\top P_{\mathbf{a}}\left(\mathbf{y}\right)P_{\mathbf{a}}\left(\mathbf{1}\right)^\top P_{\mathbf{a}}\left(\mathbf{y}\right)}{P_{\mathbf{a}}\left(\mathbf{y}\right)^\top P_{\mathbf{a}}\left(\mathbf{y}\right)} =P_{P_{\mathbf{a}}\left(\mathbf{y}\right)}\left(P_{\mathbf{a}}\left(\mathbf{1}\right)\right)^\top P_{P_{\mathbf{a}}\left(\mathbf{y}\right)}\left(P_{\mathbf{a}}\left(\mathbf{\hat{f}}\right)\right)$, Eq.~(\ref{eq:an0min1}) can also be written in the following form:
\begin{eqnarray}
\max\limits_{\rho}g\left(\rho\right)=\frac{1}{2}\left(\frac{1}{\lambda_2}-\frac{1}{\lambda_1}\right)\Bigg(-\left\|P_{P_{\mathbf{a}}\left(\mathbf{y}\right)}\left(P_{\mathbf{a}}\left(\mathbf{\hat{f}}\right)\right)\right\|_2 \left\|P_{P_{\mathbf{a}}\left(\mathbf{y}\right)}\Big(P_{\mathbf{a}}\left(\mathbf{1}\right)\Big)\right\|_2\nonumber\\
+P_{P_{\mathbf{a}}\left(\mathbf{y}\right)}\Big(P_{\mathbf{a}}\left(\mathbf{1}\right)\Big)^\top P_{P_{\mathbf{a}}\left(\mathbf{y}\right)}\left(P_{\mathbf{a}}\left(\mathbf{\hat{f}}\right)\right)-P_{\mathbf{a}}\left(\mathbf{1}\right)^\top P_{\mathbf{a}}\left(\mathbf{f}\right)\Bigg).\nonumber
\end{eqnarray}

The following theorem summarize the result for the case $\beta>0,~\alpha >0$.
\begin{theorem}
When $\mathbf{r}^\top\hat{\mathbf{f}}$ achieves its minimum value at $\beta>0$ and $\alpha > 0$, this value can be computed as:
\begin{eqnarray}
\min\limits_{\mathbf{r}}\mathbf{r}^\top\hat{\mathbf{f}}= \frac{1}{2}\left(\frac{1}{\lambda_2}-\frac{1}{\lambda_1}\right)\Bigg(-\left\|P_{P_{\mathbf{a}}\left(\mathbf{y}\right)}\left(P_{\mathbf{a}}\left(\mathbf{\hat{f}}\right)\right)\right\|_2 \left\|P_{P_{\mathbf{a}}\left(\mathbf{y}\right)}\Big(P_{\mathbf{a}}\left(\mathbf{1}\right)\Big)\right\|_2\nonumber\\
+P_{P_{\mathbf{a}}\left(\mathbf{y}\right)}\Big(P_{\mathbf{a}}\left(\mathbf{1}\right)\Big)^\top P_{P_{\mathbf{a}}\left(\mathbf{y}\right)}\left(P_{\mathbf{a}}\left(\mathbf{\hat{f}}\right)\right)-P^\top_{\mathbf{a}}\left(\mathbf{1}\right) P_{\mathbf{a}}\left(\mathbf{f}\right)\Bigg).\label{eq:an0min}
\end{eqnarray}
\end{theorem}

\begin{corollary}
When $\mathbf{r}^\top\hat{\mathbf{f}}$ achieves its minimum value at $\beta>0$ and $\alpha > 0$, the corresponding $-\min\boldsymbol{\theta}^\top\mathbf{\hat{f}}$ can be computed as:
\begin{eqnarray}
-\min\boldsymbol{\theta}_2^\top\mathbf{\hat{f}}&=&-\min{\mathbf{r}^\top\hat{\mathbf{f}}}-\mathbf{\hat{c}}^\top\hat{\mathbf{f}}\nonumber\\ &=&\frac{1}{2}\left(\frac{1}{\lambda_2}-\frac{1}{\lambda_1}\right)\Bigg(\left\|P_{P_{\mathbf{a}}\left(\mathbf{y}\right)}\left(P_{\mathbf{a}}\left(\mathbf{\hat{f}}\right)\right)\right\|_2 \left\|P_{P_{\mathbf{a}}\left(\mathbf{y}\right)}\Big(P_{\mathbf{a}}\left(\mathbf{1}\right)\Big)\right\|_2\nonumber\\
& &-~ P_{P_{\mathbf{a}}\left(\mathbf{y}\right)}\Big(P_{\mathbf{a}}\left(\mathbf{1}\right)\Big)^\top P_{P_{\mathbf{a}}\left(\mathbf{y}\right)}\left(P_{\mathbf{a}}\left(\mathbf{\hat{f}}\right)\right)-\mathbf{\hat{f}}^\top\boldsymbol{\theta}_1\Bigg).\label{eq:an0min_all}
\end{eqnarray}
\end{corollary}

\subsection{The Feature Screening Algorithm}
Algorithm~\ref{alg:main} shows the procedure of screening features for L1-Regularized L2-Loss Support Vector Machine. Given $\lambda_1$, $\lambda_2$, and $\boldsymbol{\theta}_1$, the algorithm returns a list $\mathds{L}$, which contains the indices of the features that are potential to have nonzero weights when $\lambda_2$ is used as the regularization parameter. 

For each feature, in Line 3, the algorithm weight the feature using $\mathbf{Y}$. Then, in Line 4 and Line 5, it computes $\max\left|\mathbf{\hat{f}}^\top\boldsymbol{\theta}\right|$. If the value is larger than 1, it adds the index of the feature to $\mathds{L}$ in Line 7. The function \textsf{neg\_min}$\left(\mathbf{\hat{f}}\right)$ computes $-\min\boldsymbol{\theta}_2^\top\mathbf{\hat{f}}$ using the results obtained in the preceding subsections.

Since $P_\mathbf{u}\left(-\mathbf{v}\right)=-P_\mathbf{u}\left(\mathbf{v}\right)$, it is easy to see that the intermediate results generated when computing \textsf{neg\_min}$\left(\mathbf{\hat{f}}\right)$ can be used to accelerate the computation of \textsf{neg\_min}$\left(-\mathbf{\hat{f}}\right)$. Also it is easy to verify that in the worst case, the computational cost for evaluating one feature is $O\left(n\right)$. Therefore, to evaluate all $m$ features the total computational cost is $O\left(m\times n\right)$.

\SetKwProg{Fn}{Function}{}{end}\SetKwFunction{negmin}{neg\_min}%
\SetKw{Continue}{continue}
\IncMargin{1em}
\begin{algorithm}[t]
\LinesNumbered
\SetAlgoLined
\KwIn{$\mathbf{X}\in\IR^{n\times m}$, $\mathbf{y}\in\IR^{n}$, $\lambda_1$, $\lambda_2$, $\boldsymbol{\theta}_1\in\IR^{n}$.}
\KwOut{$\mathds{L}$, the kept feature list.}

$\mathds{L}=\emptyset$, $i=1$, $\mathbf{Y}=diag\left(\mathbf{y}\right)$\;

\For{$i\le m$}{
%	\If{$\big(\left\|\mathbf{\hat{c}}\right\|+l\big)\left\|\mathbf{f}\right\|<1$}{
%		\Continue\;
%	}
	$\mathbf{\hat{f}}=\mathbf{Y}\mathbf{f}_i$\;
	$m_1$=\negmin{$\mathbf{\hat{f}}$}, $m_2$=\negmin{$-\mathbf{\hat{f}}$}\;
	$m=\max\left\{m_1,m_2\right\}$\;
	\If{$m\ge 1$} {
		$\mathds{L}=\mathds{L}\cup\left\{i\right\}$\;
	}
	$i=i+1$\;
}
\Return{$\mathds{L}$}\;
\vskip0.2in
\Fn{\negmin{$\mathbf{\hat{f}}$}}{
	\If{$\frac{P_{\mathbf{y}}\left(\mathbf{a}\right)^\top P_{\mathbf{y}}\left(\mathbf{\hat{f}}\right)}{\|P_{\mathbf{y}}\left(\mathbf{a}\right)\|\|P_{\mathbf{y}}\left(\mathbf{\hat{f}}\right)\|}=-1$}{
		compute $m$ using Eq.~(\ref{eq:b0min_all})\;
		\Return{$m$}\;
	}
	
	\If{$P_{\mathbf{y}}\left(\mathbf{a}\right)^\top\left(\frac{P_{\mathbf{y}}\left(\mathbf{b}\right)}{\left\|P_{\mathbf{y}}\left(\mathbf{b}\right)\right\|_2}-\frac{P_{\mathbf{y}}\left(\mathbf{\hat{f}}\right)}{\left\|P_{\mathbf{y}}\left(\mathbf{\hat{f}}\right)\right\|_2}\right)\le 0$}{
		compute $m$ using Eq.~(\ref{eq:a0min_all})\;
		\Return{$m$}\;	
	}
	compute $m$ using Eq.~(\ref{eq:an0min_all})\;
	\Return{$m$}\;	
}
\caption{The procedure of screening features for L1-Regularized L2-Loss Support Vector Machine (SVM).}\label{alg:main}
\end{algorithm}

\bibliographystyle{abbrv}
\bibliography{ScrSVM}
\end{document}